\providecommand{\pgfsyspdfmark}[3]{}
\PassOptionsToPackage{usenames,dvipsnames}{xcolor}

\documentclass[final,12pt,cleveref]{colt2023-arxiv} 

\usepackage{times}
\usepackage{algorithm}
\usepackage{algorithmic}
\usepackage{wrapfig}

\usepackage[utf8]{inputenc}
\usepackage[nolist,nohyperlinks]{acronym}
\usepackage{multirow}
\usepackage{enumitem}

\newcommand{\todof}[1]{}
\newcommand{\todok}[1]{}
\newcommand{\todoi}[1]{}
\newcommand{\sout}[1]{}
\newcommand{\kj}[1]{}

\renewcommand{\Pr}{\field{P}}

\newcommand{\sC}{\mathcal{C}}
\newcommand{\sX}{\mathcal{X}}

\newcommand{\field}[1]{\mathbb{#1}}

\newcommand{\R}{\field{R}}
\newcommand{\Nat}{\field{N}}
\newcommand{\E}{\field{E}}
\newcommand{\Var}{\mathrm{Var}}

\newcommand{\btheta}{\boldsymbol{\theta}}

\newcommand{\sgn}{\mbox{\sc sgn}}

\newcommand{\dt}{\displaystyle}

\newcommand{\reals}{\mathbb{R}}

\DeclareMathOperator{\Regret}{Regret}
\DeclareMathOperator{\Wealth}{Wealth}

\newcommand{\KL}[2]{\operatorname{KL}\left({#1};{#2}\right)}  

\usepackage{commath}
\usepackage{xspace}

\def\ddefloop#1{\ifx\ddefloop#1\else\ddef{#1}\expandafter\ddefloop\fi}
\def\ddef#1{\expandafter\def\csname c#1\endcsname{\ensuremath{\mathcal{#1}}}}
\ddefloop ABCDEFGHIJKLMNOPQRSTUVWXYZ\ddefloop
\def\ddef#1{\expandafter\def\csname b#1\endcsname{\ensuremath{{\boldsymbol{#1}}}}}
\ddefloop ABCDEFGHIJKLMNOPQRSUVWXYZabcdefghijklmnopqrtsuvwxyz\ddefloop
\def\ddef#1{\expandafter\def\csname h#1\endcsname{\ensuremath{\hat{#1}}}}
\ddefloop ABCDEFGHIJKLMNOPQRSTUVWXYZabcdefghijklmnopqrsuvwxyz\ddefloop 
\def\ddef#1{\expandafter\def\csname hc#1\endcsname{\ensuremath{\widehat{\mathcal{#1}}}}}
\ddefloop ABCDEFGHIJKLMNOPQRSTUVWXYZ\ddefloop
\def\ddef#1{\expandafter\def\csname r#1\endcsname{\ensuremath{\mathring{#1}}}}
\ddefloop ABCDEFGHIJKLMNOPQRSTUVWXYZabcdefghijklmnopqrtsuvwxyz\ddefloop
\def\ddef#1{\expandafter\def\csname bar#1\endcsname{\ensuremath{\bar{#1}}}}
\ddefloop ABCDEFGHIJKLMNOPQRSTUVWXYZabcdefghijklmnopqrtsuvwxyz\ddefloop
\def\ddef#1{\expandafter\def\csname wbar#1\endcsname{\ensuremath{\overline{#1}}}}
\ddefloop ABCDEFGHIJKLMNOPQRSTUVWXYZabcdefghijklmnopqrtsuvwxyz\ddefloop
\def\ddef#1{\expandafter\def\csname tc#1\endcsname{\ensuremath{\widetilde{\mathcal{#1}}}}}
\ddefloop ABCDEFGHIJKLMNOPQRSTUVWXYZ\ddefloop

\DeclareMathOperator{\EE}{\mathbb{E}}

\DeclareMathOperator{\PP}{\mathbb{P}}

\newcommand{\fr}[2]{ { \frac{#1}{#2} }}

\newcommand*\diff{\mathop{}\!\mathrm{d}}

\def\dt{{\ensuremath{\delta}\xspace} }
\def\KL{\ensuremath{\normalfont{\mathsf{KL}}}}

\newcommand{\prr}[1]{\left( #1 \right)}
\newcommand{\bra}[1]{\left[ #1 \right]}

\newcommand{\bmid}{\;\middle|\;}
\newcommand{\df}{:=}

\newcommand{\ps}{P_n}

\newcommand{\pz}{P_0}

\newcommand{\leqC}{\lesssim}

\def\kl{{\mathsf{kl}}}

\def\hmu{{\ensuremath{\hat{\mu}}} }

\def\th{\theta}

\def\NN{\mathbb{N}}
\usepackage{algorithmic}
\begin{acronym}
\acro{RLS}{Regularized Least Squares}
\acro{ERM}{Empirical Risk Minimization}
\acro{RKHS}{Reproducing kernel Hilbert space}
\acro{DA}{Domain Adaptation}
\acro{PSD}{Positive Semi-Definite}
\acro{SGD}{Stochastic Gradient Descent}
\acro{OGD}{Online Gradient Descent}
\acro{GD}{Gradient Descent}
\acro{SGLD}{Stochastic Gradient Langevin Dynamics}
\acro{IW}{Importance Weighted}
\acro{MGF}{Moment-Generating Function}
\acro{ES}{Efron-Stein}
\acro{ESS}{Effective Sample Size}
\acro{KL}{Kullback-Liebler}
\acro{SVD}{Singular Value Decomposition}
\acro{PL}{Polyak-Łojasiewicz}
\acro{NTK}{Neural Tangent Kernel}
\acro{KLS}{Kernelized Least-Squares}
\acro{KRLS}{Kernelized Regularized Least-Squares}
\acro{ReLU}{Rectified Linear Unit}
\acro{NTRF}{Neural Tangent Random Feature}
\acro{NTF}{Neural Tangent Feature}
\acro{RF}{Random Feature}
\acro{RWY}{Raskutti-Wainwright-Yu}
\acro{CW}{Celisse-Wahl}
\acro{PRM}{Penalized Reward Maximization}
\end{acronym}

\title[Tighter PAC-Bayes Bounds Through Coin-Betting]{Tighter PAC-Bayes Bounds Through Coin-Betting}

\coltauthor{%
 \Name{Kyoungseok Jang}\footnote{Authors are ordered alphabetically.} \Email{ksajks@arizona.edu}\\
 \addr University of Arizona%
 \AND
 \Name{Kwang-Sung Jun} \Email{kjun@cs.arizona.edu}\\
 \addr University of Arizona%
 \AND 
 \Name{Ilja Kuzborskij} \Email{iljak@deepmind.com}\\
 \addr DeepMind%
 \AND 
 \Name{Francesco Orabona} \Email{francesco@orabona.com}\\
 \addr Boston University%
}

\begin{document}

\maketitle

\begin{abstract}%
  We consider the problem of estimating the mean of a sequence of random elements $f(X_1, \theta)$ $, \ldots, $ $f(X_n, \theta)$ where $f$ is a fixed scalar function, $S=(X_1, \ldots, X_n)$ are independent random variables, and $\theta$ is a possibly $S$-dependent parameter.
  An example of such a problem would be to estimate the generalization error of a neural network trained on $n$ examples where $f$ is a loss function. Classically, this problem is approached through concentration inequalities holding uniformly over compact parameter sets of functions $f$, for example as in Rademacher or VC type analysis.
However, in many problems, such inequalities often yield numerically vacuous estimates.
Recently, the \emph{PAC-Bayes} framework has been proposed as a better alternative for this class of problems for its ability to often give numerically non-vacuous bounds.
In this paper, we show that we can do even better: we show how to refine the proof strategy of the PAC-Bayes bounds and achieve \emph{even tighter} guarantees.
Our approach is based on the \emph{coin-betting} framework that derives the numerically tightest known time-uniform concentration inequalities from the regret guarantees of online gambling algorithms.
In particular, we derive the first PAC-Bayes concentration inequality based on the coin-betting approach that holds simultaneously for all sample sizes. 
We demonstrate its tightness showing that by \emph{relaxing} it we obtain a number of previous results in a closed form including Bernoulli-KL and empirical Bernstein inequalities.
Finally, we propose an efficient algorithm to numerically calculate confidence sequences from our bound, which often generates nonvacuous confidence bounds even with one sample, unlike the state-of-the-art PAC-Bayes bounds.
\end{abstract}


\setlength{\abovedisplayskip}{3pt}
\setlength{\belowdisplayskip}{3pt}
\setlength{\abovedisplayshortskip}{3pt}
\setlength{\belowdisplayshortskip}{3pt}

\section{Introduction}
Suppose that $S = (X_1, \ldots, X_n) \in \sX$ are random elements distributed identically and independently from each other, on a probability space $(\sX, \Sigma(\sX), P_X)$.
For illustration, assume that $\sX = \reals$: A classical problem in probability and statistics is to quantify how quickly an average $(X_1 + \dots + X_n) / n$ converges to the mean $\E[X_1]$, and over the decades this problem was successfully attacked under various assumptions on the probability space through \emph{concentration inequalities}~\citep{boucheron2013concentration}.
The key assumption which enables these concentration inequalities to exhibit fast convergence to the mean is \emph{independence}.
However, in many learning-theoretic problems we are interested in the concentration of random elements which themselves \emph{depend on the sample} $S$, and therefore are not independent.
In this paper, we formalize the above by assuming that we are given a \emph{fixed} measurable function $f : \Theta \times \sX \to [0,1]$, where $\Theta$ is a \emph{parameter} space, and so now we are interested in the concentration of $(f(\theta, X_1) + \dots + f(\theta, X_n)) / n$ around its mean, where $\theta$ is potentially $S$-dependent.
For example, $f(\theta, X_i)$ could be the loss incurred by a learning algorithm on the $i$-th example, where the parameters $\theta$ are generated based on the sample $S$.
In the context of this example, the mean $\EE[f(\theta, X_1)]$ is called the statistical risk.

To this end, the classical approach to alleviating the dependence nuance is to derive \emph{uniform} concentration inequalities that hold \emph{simultaneously for all} parameters in a compact set $\Theta$.
For example, consider the following concentration inequality that holds with probability at least $1-\delta$,\footnote{The notation $\leqC$ hides universal constants and logarithmic factors.}
\begin{align*}
  \sup_{\theta \in \Theta} \ \left| \frac1n \sum_{i=1}^n f(\theta, X_i) - \E[f(\theta, X_1)] \right| \leqC \sqrt{\frac{\mathrm{capacity}(\Theta) + \ln \frac{1}{\delta}}{n}}~.
\end{align*}
Here the capacity term, such as VC dimension, metric entropy, or Rademacher complexity \citep{wainwright2019high}, scales with the ``size'' of the set $\Theta$.
Mentioned notions accurately capture the capacity in many learning problems, such as with linear parameterizations \citep{bartlett2002rademacher,kakade2008complexity}.
However, in some other problems, e.g., in learning with overparameterized neural networks, the pessimistic nature of uniform bounds makes them vacuous~\citep{zhang17understanding}.

On the other hand, in recent years, there has been a strong interest in the alternative to the uniform bounds, based on the \emph{PAC-Bayes} analysis~\citep{McAllester98}, which, remarkably, on some instances demonstrates \emph{non-vacuous} bounds for the generalization ability of deep learning algorithms~\citep{dziugaite2017computing,perez-ortiz2020tighter,zhou2019nonvacuous}.
%
%
%
In the PAC-Bayes analysis instead of taking $\sup_{\theta \in \Theta}$ as in the uniform approach, we assume that the parameters $\theta$ are now \emph{random} and follow a \emph{data-dependent}, so-called, \emph{posterior} distribution $\ps$.
In this paper, we are interested in estimating an expected mean $\int \mu_{\theta} \diff \ps(\theta)$ with $\mu_{\theta} \df \E[f(X_1, \theta) \mid \theta]$, uniformly over all data-dependent posteriors $\ps$.
This setting covers the one considered in the PAC-Bayes literature~\citep{alquier2021user} where usually the function $f$ represents the loss of a predictor parameterized by $\theta$ from a parameter space $\Theta$.
In this view, we can think of $\mu_{\theta}$ as the risk of a predictor parameterized by $\theta$, while $\int \mu_{\theta}  \diff \ps(\theta)$ is the risk of a randomized predictor that uses a random $\theta$ drawn from the distribution $\ps$.
The second important component of the PAC-Bayes model is a \emph{prior} distribution $\pz$ which does not depend on $S$ and captures our prior belief about the inductive bias in the problem instance.
A basic PAC-Bayes concentration inequality~\citep{McAllester98} then takes the form of
\begin{align*}
  \Bigg| \int \frac1n \sum_{i=1}^n f(\theta, X_i) \diff \ps - \int \mu_{\theta}  \diff \ps(\theta) \Bigg| \leqC \sqrt{\frac{\KL(\ps \| \pz) + \ln \frac{1}{\delta}}{n}}
\end{align*}
with probability at least $1-\delta$, where instead of the capacity term we have a \ac{KL} divergence between the posterior and the prior.
Numerically speaking, PAC-Bayes bounds tend to give much tighter bounds than their uniform counterparts, largely because the $\KL$ term is typically smaller than the capacity term (such as VC dimension), for an appropriate (user's) choice of $\ps, \pz$.
\paragraph{Our contributions: PAC-Bayes meets coin-betting}
In this paper, we show that we can obtain \emph{even tighter} PAC-Bayes bounds using recent advances in the theory of concentration inequalities through gambling algorithms~\citep{OrabonaJ21}.
In particular, we show that it is possible to obtain a new \emph{coin-betting} based PAC-Bayes bound that directly \emph{implies} a number of previous results.
Moreover, numerically evaluating this new upper bound, we show that it is numerically tighter than all previous approaches.
In \cref{sec:results} we present our main result, a concentration inequality of the following form, which holds with probability at least $1-\delta$, \emph{simultaneously for all} $n \in \mathbb{N}$, \emph{all} data-dependent posteriors $\ps$, and all data-free priors $\pz$:
\begin{align}
\label{eq:intromain}
\int \max_{\lambda \in [-\frac{1}{1-\mu_{\theta}}, \frac{1}{\mu_{\theta}}]} \ \sum_{i=1}^n \ln\prr{1 + \lambda (f(\theta,X_i) - \mu_{\theta})} \diff \ps(\theta)
  \leq
  \KL(\ps \| \pz)
  +
  \ln \frac{c \, \sqrt{n}}{\delta}~.
\end{align}
Moreover, the confidence interval for $\int \mu_{\theta} \diff \ps(\theta)$ is obtained by solving the optimization problems
\begin{align}\label{eq:intro-optim}
  \Bigg[\min_{(\mu_\theta)_{\th\in\Theta} \in M }  \int \mu_\theta  \diff \ps(\theta), \,\, \max_{(\mu_\theta)_{\th\in\Theta} \in M}  \int \mu_\theta  \diff \ps(\theta) \Bigg],
\end{align}
where $M$ is a class of mean functions $(\mu_\th)_{\th\in\Theta}$ that satisfy the constraint given by \cref{eq:intromain}, which is the first of its kind in the PAC-Bayes literature.
We formalize these optimization problems in \cref{sec:howtocompute} and show that the constraint is \emph{convex}, making them efficiently solvable in some cases.
In \cref{sec:experiments} we also experimentally validate our approach.

In addition, we show that \cref{eq:intromain} is tighter than some well-known PAC-Bayes inequalities, such as McAllester's inequality~\citep{McAllester98}, Maurer's inequality for Bernoulli \ac{KL} divergence~\citep{maurer2004note}, and PAC-Bayes empirical Bernstein's inequality~\citep{tolstikhin2013pac}.
This is done by \emph{relaxing} inequality in \cref{eq:intromain} by simple lower bounds of the logarithmic term.
We show that even relaxing \cref{eq:intromain} leads to a tighter bound than Maurer's one that is known to be very tight numerically.

The observation above implies that from our result we can derive all these versions, obtain bounds, and then take an intersection of all of them \textit{without having to split $\dt$}.
This is in stark contrast to empirical Bernstein's bounds~\citep{tolstikhin2013pac} that are often numerically looser than KL bounds, while being orderwise tighter than KL bounds like Maurer's one.
Attempting to take an intersection with KL bounds requires splitting $\dt$, which undesirably inflates the bound.
This is not the case for our method -- our result can be seen as ``the right'' type of concentration inequality that is superior to the rest up to constant factors inside an additive logarithmic term.
Finally, \cref{eq:intromain} (and all its corollaries) holds simultaneously for all sample sizes,
delivering time-uniform PAC-Bayes confidence sequences.

\paragraph{Organization of the paper}
After a discussion of related work (\cref{sec:related}) and notations (\cref{sec:def}), in \cref{sec:betting} we briefly present the idea behind concentration through coin-betting.
In \cref{sec:results} we present our main results, discuss some implications, and include the proof of a new concentration inequality in \cref{sec:pac-bayes-cb-proof}.
In \cref{sec:howtocompute}, we discuss how to compute our concentration inequality numerically (without any relaxations).
Finally, in \cref{sec:experiments} we present numerical simulations comparing our inequality to a number of baselines from the PAC-Bayes literature.
\section{Related Work}
\label{sec:related}
\paragraph{Concentration from coin-betting}
The coin-betting formalism considered here (see \Cref{sec:betting}) goes back to \citet{Ville39} and Kelly betting system~\citep{Kelly56} and has an intimate connection to the Universal Portfolio theory~\citep{Cover91}.
Building on the ideas of \citet{Ville39}, \citet{ShaferV05} introduced a general framework aiming at giving a foundation to the theory of probability rooted in gambling strategies. However, their framework is very general and it does not suggest specific methods to construct the betting strategies.
The first paper to introduce the idea of using the regret of online betting algorithms to produce new concentration inequalities was in \citet{JunO19}, which in turn builds on \citet{RakhlinS17} that showed the equivalence between the regret guarantees of generic online linear algorithms and martingale tail bounds.
\paragraph{PAC-Bayes}
Since the introduction of PAC-Bayes bounds by \citet{McAllester98},
there has been significant growth and development in both theory and applications; see \citet{alquier2021user} for a comprehensive survey.
Early papers focused on tightening the bound of \citet{McAllester98}, which can be seen as a PAC-Bayes version of Hoeffding's inequality.
In particular, \citet{LangfordCaruana2001,seeger2002,maurer2004note} focused on the setting of a binary classification where the goal is to bound $\KL$ divergence between Bernoulli distributions.
Such bounds are tighter than the mere difference of the risk and empirical risk due to Pinsker's inequality, and the numerically tightest known inequality within this group is Maurer's inequality \citep{maurer2004note} (see for instance experiments of \citet{MhammediGG19}).
In this paper, we recover the result of \citet{maurer2004note} by relaxing \cref{eq:intromain}.

Towards data-dependent bounds, \citet{tolstikhin2013pac} adapted an empirical Bernstein's inequality~\citep{audibert2007tuning,maurer2009empirical} to the PAC-Bayes setting, making generalization bounds variance dependent.
Once again, we recover the PAC-Bayes empirical Bernstein's inequality by relaxing our main result of \cref{eq:intromain} without any plug-in arguments, through a relatively straightforward proof.
Several works went further in making bounds data-dependent by manipulating the $\KL$ term.
\citet{ambroladze2006tighter} explored the idea of splitting the sample and deriving the \emph{prior} from a held-out sample while obtaining the posterior from the remaining part.
This technique proved very fruitful in making PAC-Bayes bounds much tighter.
Indeed, recent non-vacuous generalization bounds for deep neural networks are largely attributed to this technique~\citep{dziugaite2018entropy,perez-ortiz2020tighter}.
Clearly, the results developed in this paper can be readily applied together with the splitting technique.
The splitting technique was also investigated beyond the $\KL$ term.
In particular, \citet{MhammediGG19,wu2022i} developed intricate bounds akin to empirical Bernstein's inequalities where the splitting is done with respect to the sample variance (two variance terms) in addition to the $\KL$ term.
These are among the numerically tightest known PAC-Bayes bounds.
However, due to their highly problem-dependent nature, it is challenging to compare these bounds theoretically.

The proof of our main result relies on showing that the exponential moment of the optimal log-wealth with respect to $\theta$ is a martingale.
Several papers in PAC-Bayes literature have shown results exploiting (super-)martingale concentration, which allowed them to
relax the independence assumption in the data sequence~\citep{seldin2012pac} or to replace unboundedness of $f()$ by weaker assumptions~\citep{kuzborskij2019efron,haddouche2022pac}.
To this end, \citet{haddouche2022pac} exploited Ville's inequality (as in our proof), which allowed them to show a bound that holds uniformly over $n \in \mathbb{N}$.

Finally, it is known that solving the classical PAC-Bayes bound of \citet{McAllester98} for the posterior results in a \emph{Gibbs} posterior $\ps(\theta) \propto e^{- \frac1n \sum_i f(\theta, X_i)} \diff \pz(\theta)$.
A large body of literature has looked at learning-theoretic properties
of Gibbs predictors~\citep{catoni2007,alquier2016properties,raginsky2017non,kuzborskij2019distribution,grunwald2019tight}.
The concentration inequality we develop here (\cref{eq:intromain}) is of a very different shape compared to \citep{McAllester98}, though it can be easily relaxed to obtain it.
As such, the Gibbs predictor might be a \emph{suboptimal} solution to \cref{eq:intromain}, and it is an interesting open problem to characterize such a solution.

\section{Definitions}
\label{sec:def}
We denote by $(x)_+ = \max\{x, 0\}$.
If $P$ and $Q$ are probability measures over $\Theta$
such that $P \ll Q$, the \ac{KL} divergence between $P$ and $Q$ is defined as
$\KL(P, Q) \df \int P(\diff x) \ln \frac{\diff P}{\diff Q}(x)$.
With a slight abuse of notation, we also write $\KL(p \| q)$ where $p = \diff P/\diff \lambda$ and $q=\diff Q/\diff \lambda$ are densities of $P$ and $Q$ with respect to some common $\sigma$-finite measure $\lambda$.
If a set $\sX$ is uniquely equipped with a $\sigma$-algebra, the underlying $\sigma$-algebra will be denoted by $\Sigma(\sX)$.
We formalize a ``data-dependent distribution'' through the notion of a \emph{probability kernel}~\citep[see, e.g.,][]{kallenberg2017} which is defined as a map $K : \sX^n \times \Sigma(\Theta) \rightarrow [0,1]$ such that
for each $B \in \Sigma(\Theta)$ the function $s \mapsto K(s,B)$ is measurable and for each $s \in \sX^n$ the function $B \mapsto K(s,B)$ is a probability measure over $\Theta$.
We write $\cK(\sX^n,\Theta)$ to denote the set of all probability kernels from $\cX^n$ to distributions over $\Theta$.
In that light, when $P \in \cK(\sX^n,\Theta)$ is evaluated on $S \in \sX^n$ we use the shorthand notation $\ps = P(S)$.
\section{Warm-up: From Betting To Concentrations}
\label{sec:betting}
In this section, we briefly explain how to obtain new concentration inequalities from betting algorithms, following \citet{OrabonaP16,RakhlinS17, JunO19}.

Let $c_t \in [-1, 1]$ be a sequence of ``continuous coin'' outcomes chosen arbitrarily.
In each round, the bettor bets $|x_t|$ money on the outcome $\sgn(x_t)$. Then, $c_t$ is revealed and the bettor wins/loses $x_t c_t$ money. Define the initial wealth $\Wealth_0\df 1$ and the wealth at the end of round $t$ as
\[
\Wealth_t 
\df \Wealth_{t-1} + c_t x_t
= 1 + \sum_{s=1}^t c_s x_i~.
\]
We also assume that the algorithm guarantees $\Wealth_t\geq 0$, hence we must have $x_t \in [-\Wealth_{t-1},$ $\Wealth_{t-1}]$.
Given that no assumptions are made on how $c_t$ is generated, this is essentially an online game~\citep{Cesa-BianchiL06,Orabona19}. So, our aim is to achieve an amount of money close to the one of a fixed comparator. In particular, let $\Wealth_t(\lambda)$ be the wealth obtained by a bettor that bets $\lambda \Wealth_{t-1}(\lambda)$ in round $t$ with initial wealth equal to $\Wealth_0(\lambda)\df 1$ and
\[
\Wealth_t(\lambda)
\df \Wealth_{t-1}(\lambda) + c_t \lambda \Wealth_{t-1}(\lambda)
= \prod_{s=1}^t (1+c_s \lambda)~.
\]
We can now formally define the regret of the betting algorithm as 
\[
\Regret_T
\df \frac{\max_{\lambda \in [-1,1]} \Wealth_{T}(\lambda)}{\Wealth_T}~.
\]
It is well-known that it is possible to design optimal online betting algorithms where the regret is polynomial in $T$~\citep[Chapters 9 and 10]{Cesa-BianchiL06}.

\paragraph{Closed form concentration, following \citet{RakhlinS17}}
Here, we summarize the basic idea of \citet{RakhlinS17} used to obtain concentration inequalities from online learning algorithms, specializing it to online betting algorithms as in \citet{JunO19}.

Consider $X_t$ to be a sequence of i.i.d. random variables supported on $[0,1]$ such that $\E[X_t]=\mu$.
Set $c_t=X_t-\mu \in [-1,1]$, so that regardless of the online betting algorithm we have $\E[\Wealth_t]=1$. Also, assume that $\Regret_T\leq R(T)$, where $R:\Nat \rightarrow \R_+$.
Let's now lower bound $\Wealth_T(\lambda)$ to obtain a familiar quantity.
Using the inequality $1+x \geq \exp(x-x^2)$ for $x \geq -0.68$, we obtain
\begin{align*}
&\max_{\lambda \in [-1,1]} \ \Wealth_{T}(\lambda)
\geq \max_{\lambda \in [-1/2,1/2]} \Wealth_{T}(\lambda)
= \max_{\lambda \in [-1/2,1/2]} \ \prod_{t=1}^T (1+\lambda (X_t-\mu)) \\
&\geq \max_{\lambda \in [-1/2,1/2]} \exp \left(\lambda \sum_{t=1}^T(X_t-\mu -\lambda (X_t-\mu)^2)\right)
\geq \max_{\lambda \in [-1/2,1/2]} \exp\left(\lambda \sum_{t=1}^T (X_t-\mu -\lambda )\right).
\end{align*}
Putting it all together and using Markov's inequality, for any $\lambda \in [-1/2,1/2]$ we get
\[
\Pr\left\{ \lambda \sum_{t=1}^T(X_t-\mu -\lambda) \geq \ln \frac{R(T)}{\delta}\right\}
\leq \Pr\left\{ \Wealth_T(\lambda) \geq \frac{R(T)}{\delta}\right\}
\leq \Pr\left\{  \Wealth_T \geq  \frac{1}{\delta}\right\}
\leq \delta.
\]
Choosing $\lambda$ with the proper sign and of the order of $\frac{1}{\sqrt{T}}$, we get roughly Hoeffding inequality when the Regret is $\cO(1)$, which is possible for fixed $T$ and for this specific lower bound to the optimal wealth. Even better concentrations can be obtained carrying around the $(X_t-\mu)^2$ terms, resulting in an empirical Bernstein-style bound.

It is important to stress that we do not need to run the betting algorithm to obtain the concentration. Instead, we only need the \emph{existence} of a betting algorithm and its associated regret guarantee.

\paragraph{Tighter concentration inequalities}
From the above reasoning, it should be clear that we can obtain a tighter bound by giving up the closed-form expression by avoiding to lower bound the wealth:
\[
\Pr\left\{ \max_{\lambda \in [-1,1]} \ \sum_{t=1}^T\ln(1+\lambda  (X_t - \mu))
\geq  \ln\frac{R(T)}{\delta}\right\}
=\Pr\left\{ \max_{\lambda \in [-1,1]} \Wealth_{T}(\lambda) \geq \frac{R(T)}{\delta}\right\} 
\leq \delta~.
\]
In this case, we can numerically invert this inequality and obtain a tighter concentration.

Now, we depart from \citet{RakhlinS17} and, instead of using Markov's inequality, we follow \citet{JunO19} using Ville's inequality (\Cref{thm:ville}). We can do it because, by the assumptions on the betting algorithm, the wealth is a non-negative martingale.
The use of Ville's inequality gives the uniformity over time for free and gives us a high-probability time-uniform concentration inequality.
Namely, with probability at least $1-\delta$, we have
\begin{align}
\label{eq:coinbettingineq}
  \max_t \max_{\lambda \in [-1,1]} \sum_{s=1}^t \ln ( 1+ \lambda (X_s - \mu) ) \leq \ln\frac{R(t)}{\delta}~.
\end{align}
Note that to obtain upper and lower bounds for $\mu$ it is enough to find the set of values of $\mu$ that satisfies \cref{eq:coinbettingineq}. This can be done efficiently because the argument of the max can be proved to be a quasi-convex one-dimensional function in $\mu$~\citep{OrabonaJ21}.
The concentration inequality above can be seen as a tight and implicit version of the empirical Bernstein's inequality for bounded random variables, just like how the KL-divergence concentration inequality is an implicit and tight version of the Bernstein's inequality for Bernoulli random variables.

\section{Main Results}
\label{sec:results}
The concentration inequality of \cref{eq:coinbettingineq} holds for i.i.d.\ random variables $S=(X_1, \ldots, X_n)$.
However, in many learning-theoretic applications, we are interested in providing confidence intervals for the mean of some data-dependent function (such as the generalization error).
To this end, in this section, we explore a scenario where $X_1, \ldots, X_n$ are replaced by a sequence $f(\theta, X_1), \ldots, f(\theta, X_n)$
such that $f$ is a fixed scalar function and $\theta$ is a \emph{data-dependent} parameter.
Clearly, elements of such a sequence are not independent, since dependence is introduced through parameter $\theta$.
Following the \emph{PAC-Bayes} viewpoint~\citep{McAllester98,alquier2021user}, $\theta$ is now \emph{random} and distributed according to some user-chosen data-dependent distribution $\ps$ called \emph{posterior}.
In addition, unlike in the traditional PAC-Bayes literature, our ours hold uniformly \emph{not only} in $\ps$, but also in the \emph{sample size} $n$.
Thus, we construct a high-probability \emph{PAC-Bayes confidence sequence}.

The next theorem, proved in \cref{sec:pac-bayes-cb-proof}, is the main result of our paper, which generalizes the concentration analysis of \cref{sec:betting} to the PAC-Bayes setting.
\begin{theorem}
\label{thm:pac-bayes-cb}  
Let $S=(X_1, \ldots, X_n)$ be a tuple of i.i.d.\ random variables taking values in some measurable space $\sX$.
Let $\ps$ be a data-dependent distribution over some measurable space $\Theta$
and let $P_0$ be any probability measure over $\Theta$
independent from sample $S$.
Let $f : \Theta \times \sX \rightarrow [0,1]$ be any fixed measurable function,
let its mean be denoted by $\mu_{\theta} = \E[f(\theta, X_1)]$, and introduce
\begin{align*}
\psi_n^{\star}(\theta, \mu_{\theta})
  \df \max_{\lambda \in [-\frac{1}{1-\mu_{\theta}}, \frac{1}{\mu_{\theta}}]} \ \sum_{i=1}^n \ln\prr{1 + \lambda (f(\theta,X_i) - \mu_{\theta})},
  \qquad (\theta \in \Theta, \mu_{\theta} \in [0,1])~.
\end{align*}  
Then, for all $\pz$,
with probability at least $1-\delta$ for any $\delta \in (0,1]$, we have\footnote{Here $\exists n \in \mathbb{N}~, \exists P_n$ is a shorthand notation for $\exists n \in \mathbb{N}~, \exists P \in \cK(\sX^n, \Theta)$ where $\cK()$ is a set of probability kernels as defined in \Cref{sec:def}.}
\begin{align}
\label{eq:pac-bayes-cb}
    \Pr\left\{
    \exists n \in \mathbb{N}~, \exists P_n :  \int
  \psi_n^{\star}(\theta, \mu_{\theta}) \diff \ps(\theta)
  - \KL(\ps \| \pz)
  - \ln \frac{\sqrt{\pi} \, \Gamma(n+1)}{\Gamma(n+\frac12)} \geq \ln \frac{1}{\delta} \right\}
  \leq \delta~.
\end{align}
\end{theorem}
Now we discuss some of the implications of \cref{eq:pac-bayes-cb} and compare it to existing PAC-Bayes results.
The important feature of \cref{eq:pac-bayes-cb} is that it holds simultaneously for all posterior distributions, so we can freely choose the one that depends on the data.
At the same time, \cref{eq:pac-bayes-cb} is similar in shape to the concentration inequality of \cref{eq:coinbettingineq}.
In particular, $\psi_n^{\star}$ is an optimal log-wealth discussed in \cref{sec:betting}, while $\Gamma(n+1)/\Gamma(n+1/2) \sim \sqrt{n}$ is the regret bound ($R(n)$) of a certain betting algorithm.
Observe that unlike \cref{eq:coinbettingineq}, the left-hand side of the inequality is now integrated over $\theta \sim \ps$, and the term $\KL(\ps \| \pz)$ appears on the right-hand side.
In particular, the term $\KL(\ps \| \pz)$ captures the capacity of the class of posterior distributions with respect to the prior $\pz$, and it is a standard component in PAC-Bayes analyses.
\paragraph{Obtaining known PAC-Bayes inequalities by relaxing \cref{eq:pac-bayes-cb}}
By \emph{relaxing} \cref{eq:pac-bayes-cb}, we demonstrate that \cref{thm:pac-bayes-cb} gives a tighter concentration inequality compared to some inequalities in PAC-Bayes literature (proofs are deferred to \cref{sec:proofs}).
Importantly, our results extend these bounds as our relaxations hold uniformly over $n \in \mathbb{N}$, whereas previous results hold for a fixed $n$.
Abbreviate
\begin{align*}
  \hat{\mu}_{\theta} \df \frac1n \sum_{i=1}^n f(\theta, X_i) \qquad \text{ and } \qquad
  \sC_{n} \df \KL(\ps \| \pz) + \ln \frac{\sqrt{\pi} \, \Gamma(n+1)}{\Gamma(n+\frac12)}~.
\end{align*}
As a basic sanity-check, we first recover a classical result of \citet{McAllester98} through the elementary inequality $\ln(1+x) \geq x - x^2$ for $x \geq - 0.68$ (similarly as in \cref{sec:betting}), proof in \cref{sec:mcallister}.
\begin{proposition}[McAllester's inequality]
  \label{prop:mcallister}
  Set $\delta \in (0,1]$. Under conditions of \cref{thm:pac-bayes-cb}, for all priors $\pz$,
  with probability at least $1-\delta$ over the sample $S$,
  for all $n \in \mathbb{N}$ and for all data-dependent distributions $\ps$ simultaneously we have
  \begin{align*}
  \abs{\int \mu_{\theta} \diff \ps(\th) - \int \hat{\mu}_{\theta} \diff \ps(\th)}
  \leq
    2 \sqrt{\frac{\sC_n + \ln \frac{1}{\delta}}{n}}~.
\end{align*}
\end{proposition}
Note that, up to constants, the above matches the result of \citet{McAllester98},
and \emph{extends} it --- now the bound holds \emph{simultaneously} for all $n \in \mathbb{N}$.

Now we turn our attention to a type of PAC-Bayes inequality, where we the bound is given on a $\KL$ divergence between Bernoulli distributions.
Such bounds are useful in a setting of a binary classification, where the parameter of a Bernoulli distribution models a conditional probability of a positive class label.
In particular, relaxing \cref{eq:pac-bayes-cb} gets a well-known inequality of \citet{maurer2004note}:
\begin{proposition}[Maurer's inequality]
  \label{prop:maurer}
  For $p,q \in [0,1]$ let $\kl(p,q) \df p \ln(p/q) - (1-p) \ln((1-p) / (1-q))$, i.e., the $\KL$ divergence between Bernoulli distributions with parameters $p$ and $q$ respectively.
  Set $\delta \in (0,1]$. Under the conditions of \cref{thm:pac-bayes-cb}, for all priors $\pz$,
  with probability at least $1-\delta$ over the sample $S$,
  for all $n \in \mathbb{N}$ and for all data-dependent distributions $\ps$ simultaneously,
    \begin{align*}
    \kl\Big(\int \hat{\mu}_{\theta} \diff \ps(\theta), \int \mu_{\theta} \diff \ps(\theta) \Big) \leq \frac{\sC_{n} + \ln \frac{1}{\delta}}{n}~.
  \end{align*}
  \end{proposition}
The above inequality matches Maurer's bound up to a constant inside a logarithmic factor.
Furthermore, the proof of \Cref{prop:maurer} in \Cref{sec:proofmaurer} reveals that even relaxing \cref{thm:pac-bayes-cb} to have $\int \kl(\hmu,\mu_\th) \dif P_n(\th)$ in place of $\psi^\star_n(\th,\mu_\th)$ on the LHS results in a bound that is tighter than Maurer's inequality.
We confirm this numerically in Section~\ref{sec:experiments}.

We now consider a more sophisticated, sample variance-dependent concentration inequality, which exhibits a faster rate of order $1/n$ whenever the sample variance is sufficiently small.
In the non-PAC Bayes form, such a \emph{empirical Bernstein's inequality} was shown by \citet{audibert2007tuning,maurer2009empirical}, whereas the PAC-Bayes version was first presented by \citet{tolstikhin2013pac}.
The following result recovers their result up to constants through a much simpler proof by relaxing \cref{thm:pac-bayes-cb}:
\begin{proposition}[PAC-Bayes empirical Bernstein's inequality]
  \label{thm:empiricalbernstein}
  Set $\delta \in (0,1]$.
    Introduce   \begin{align*}    
  \hat{V}(\ps) \df \frac1n \int \sum_{i=1}^n (f(\theta,X_i) - \hat{\mu}_{\theta})^2 \diff \ps(\theta)~, \qquad \sC_{n,\delta} \df \sC_n + \ln \frac{1}{\delta}~.
  \end{align*}
  Under the conditions of \cref{thm:pac-bayes-cb}, for all priors $\pz$,
  with probability at least $1-\delta$ over the sample $S$,
  for all $n \in \mathbb{N}$ and for all data-dependent distributions $\ps$ simultaneously we have
\begin{align*}
  \abs{\int \mu_{\theta} \diff \ps(\th) - \int \hat{\mu}_{\theta} \diff \ps(\th) }
  \leq
  \frac{\sqrt{2 \, \sC_{n,\delta} \, \hat{V}(\ps)}}{\big(\sqrt{n} - \frac{2}{\sqrt{n}} \, \sC_{n,\delta}\big)_+}
  +
  \frac{2 \, \sC_{n,\delta}}{\big(n - 2 \, \sC_{n,\delta}\big)_+}~.
\end{align*}
\end{proposition}
Note that the inequality is fully empirical and non-vacuous as long as $\sC_{n,\delta} \leq n/2$ --- similar (empirically verifiable) requirement is also present in \citep[Theorem 4]{tolstikhin2013pac}.
Clearly, the fact that we relaxed \Cref{thm:pac-bayes-cb} to get \Cref{thm:empiricalbernstein} implies that our inequality is tighter.

\subsection{How to compute confidence intervals from \cref{eq:pac-bayes-cb} numerically}
\label{sec:howtocompute}
So far we discussed analytically computable relaxations of our inequality.
Now we turn our attention to numerical computation of \cref{eq:pac-bayes-cb} which does not require any relaxation.
Given a concrete posterior and prior pair $(\ps, \pz)$, we propose
to obtain confidence bounds for the mean $\int \mu_{\theta} \diff \ps(\theta)$ by solving the following optimization problem:
\begin{proposition}
\label{prop:opt}
  Set $\delta \in (0,1]$. Consider the optimization problem
            \begin{align}
    M_U = \max_{ \{\mu_\theta : \theta\in \text{supp}(\ps)\}} \ \int \mu_\theta  \diff \ps(\theta) \quad
    \text{ subject to } \quad \int \psi_n^{\star}(\theta,\mu_\theta) \diff \ps(\theta) \le \sC_n + \ln \frac{1}{\delta}~, \label{eq:main-invert}
  \end{align}
where $\sC_n + \ln \frac{1}{\delta}$ is the right hand side of \cref{eq:pac-bayes-cb}.
Moreover, let $M_L$ be obtained by replacing $\max$ with $\min$.
Then, under the conditions of \cref{thm:pac-bayes-cb} and with probability at least $1-\delta$, we have
\begin{align*}
  M_L
  \leq
  \int \mu_{\theta} \diff \ps(\theta) \leq M_U~.
\end{align*} 
\end{proposition}
In other words, the optimization in \cref{eq:main-invert} is carried out over the class of means of a given distribution,
and the solution gives us a valid confidence interval since \cref{thm:pac-bayes-cb} holds for any data-dependent posterior, and so it must hold for some posteriors with means within the class.
Moreover, surprisingly enough, the optimization problem is \emph{convex} since $\psi_n^{\star}(\theta, \mu_{\theta})$ appearing in the constraint is convex in $\mu_{\theta}$, thanks to the following lemma proven in \Cref{sec:maxlogwealthisconvex}.
\begin{lemma}[Convexity of the constraint]
  \label{lem:maxlogwealthisconvex}
Let $c \in [0,1]$.
Define $f(x) = \max_{-\frac{1}{1-x}\leq\lambda \leq \frac{1}{x}}\ \ln(1+\lambda (c-x))$. Then, $f(x)$ is convex for any $x \in [0,1]$.
\end{lemma}
In \cref{sec:experiments} we present synthetic experiments validating the numerical tightness of the confidence intervals obtained by solving the problem in \Cref{prop:opt}.
\subsubsection{Monte Carlo approximation of the integral}\label{subsubsec: monte-carlo approximation}
The confidence intervals of \Cref{prop:opt} can be obtained efficiently as long as we can efficiently compute or estimate integrals over parameters.
When $\Theta$ is finite we can clearly replace integrals by summations. On the other hand, for continuous (or prohibitively large finite) $\Theta$ we can
employ a Monte Carlo approximation of the integral. In particular, we can use the procedure in Algorithm~\ref{alg:mcapprox}.

\textfloatsep=.5em
\begin{algorithm}
\caption{Monte Carlo Approximation}
\label{alg:mcapprox}
\begin{algorithmic}[1]
\STATE \textbf{Input}: Failure probability $\dt$, sample size parameters $K,m \in \NN$, posterior $\ps$, prior $\pz$
\STATE Sample $K$ tuples independently $(\th_i)_{i \in B_k} \sim \ps^m$ for $k \in [K]$ where\\ $B_k = \{(k-1)m+1, \ldots, k m\}$
\STATE Solve the following optimization problem for every $k\in[K]$:
\begin{align}
  \bar\nu_U(k) := \max_{ \{\nu_{\th_i}\}_{i\in B_k}}   &~~  \fr1m\sum_{i\in B_k} \nu_{\th_i}\label{eq:optimub0}
  ~~~ \text{ s.t. } ~~~ \frac{1}{em}\sum_{i\in B_k} \psi_n^{\star}(\th_i,\nu_{\th_i} ) \le  \sC_n + \ln \frac{1}{\delta}
\end{align}
\vspace{-1em}
    \STATE Repeat the above while replacing $\max$ with $\min$. Call the resulting optimal objective function as $\bar\nu_L(k)$, $\forall k\in[K]$ \STATE Let $\hat k_U = \arg\max_{k\in[K]} \bar \nu_U(k)
$ and $\hat k_L = \arg\min_{k\in[K]} \bar \nu_L(k)$
\STATE Compute
\vspace{-.5em}
\begin{align}\label{eq:MU}
    M_U \!=\! \max\left\{\mu\!: \kl\left( \bar\nu_U(\hk_U), \mu\right)\le\! \frac{\ln\fr{K}{2\delta}}{m}\right\},\,
    M_L \!=\! \min\left\{\mu\!: \kl\left( \bar\nu_L(\hk_L), \mu\right)\le\! \frac{\ln\fr{K}{2\delta}}{m}\right\}
\end{align}
\STATE \textbf{Output:} $M_L$ and $M_U$
\end{algorithmic}
\end{algorithm}
The following proposition (proved in \Cref{sec:mc}) states its correctness.
\begin{proposition}
  \label{prop:mc}
  Set $\delta \in (0,1]$. Under the assumptions of \cref{thm:pac-bayes-cb}, let $K = \lceil \ln(1/\delta) \rceil$.
            Then, with probability at least $1-3 \delta$, the outputs $M_L$ and $M_U$ of Algorithm~\ref{alg:mcapprox} satisfy
\begin{align*}
  M_L
  \leq
  \int \mu_{\theta} \diff \ps(\theta)
  \leq
  M_U~.
\end{align*} 
\end{proposition}
Algorithm~\ref{alg:mcapprox} works by carefully controlling the Monte Carlo approximation through the deviation of the sample average over parameters from the integral.
In fact, while this is straightforward for bounded random variables, $\psi_{n}^{\star}$ considered here is not bounded.
One may attempt to make it bounded by clipping $\psi^*_n(\th_i,\mu_{\th_i})$ or reducing the range of $\lambda$ in the max operator in the definition of $\psi^*_n$, but these both lead to nonconvex constraints in \cref{eq:optimub0}.
Alternatively, since \cref{eq:main-invert} suggests that we need to lower bound an integral, we could right away get a ``low-probability'' bound arising from Markov's inequality:
$
  \frac{1}{\delta} \int \psi_{n}^{\star}(\theta,\mu_{\theta}) \diff \ps(\th)
  \geq
  \frac1m \sum_{i \in B_k} \psi_{n}^{\star}(\theta_i,\mu_{i})~.
$
Since this is unsatisfactory, here we resort to the ``boosting-the-confidence'' method~\citep{schapire1990strength,shalev2010learnability} which allows to convert polynomial concentration bounds into exponential ones at the expense of sample partitioning and running the algorithm multiple times, as described in \Cref{prop:mc}.
  Note that in our case this just translates into extra computation (running Monte Carlo approximation on $K$ independent parameter tuples), because we can always sample more parameter observations from $\ps$.
\Cref{prop:mc} is then justified through the use of the following inequality shown in \Cref{sec:mcboost}:
\begin{proposition}
  \label{prop:mcboost}
  Under conditions of \Cref{prop:mc}, with probability at least $1 - e^{-K}$,
\begin{align}
  \label{eq:minconstraint}
  \int \psi_n^{\star}(\th, \mu_{\th}) \diff \ps(\th) \geq \min_{k \in [K]} \ \frac{1}{e m} \sum_{i \in B_k} \psi_n^{\star}(\th_i, \mu_{\th_{i}})~.
\end{align}
\end{proposition}
\paragraph{How large $m$ needs to be?}
One question not discussed so far is the choice of Monte Carlo sample size $m$.
Technically, \Cref{prop:mcboost} holds for any $m \in \mathbb{N}$, but we can expect that choosing small $m$ will result in overly loose constraints in \cref{eq:minconstraint} and so the final confidence intervals will be wide.
To gauge a good choice of $m$ we consider a lower tail Bernstein's inequality~\citep{maurer2003bound}, which lower bounds the left-hand side of the constraint~\cref{eq:minconstraint}:
\begin{equation}
  \label{eq:psisumlowertail}
  \frac1m \sum_{i=1}^m \psi_n(\th_i, \mu_{\th_{i}})
  >
  \prr{
    \int \psi_n^{\star}(\th, \mu_{\th}) \diff \ps(\theta) - \sqrt{\frac{2 \ln \frac{1}{\delta}}{m}  {\color{blue} \int \psi_n^{\star}(\th, \mu_{\th})^2 \diff \ps(\theta)}}
  }_+~.
\end{equation}
Thus, having a {\color{blue} raw second moment} of $\psi_n^{\star}$ of order $o(m)$ guarantees asymptotic convergence of the sample average to the integral in the constraint \eqref{eq:minconstraint}.
Having a finite raw second moment suggests that the constraint is tight and a reasonable choice, is, for instance, $m = n^2$.
On the other hand, for ``hard'' problems (e.g., heavy-tailed) such moment is infinite and Monte Carlo estimation is infeasible.

\subsection{Proof of \cref{thm:pac-bayes-cb}}
\label{sec:pac-bayes-cb-proof}
  Let $\Delta_i(\theta) \df f(\theta, X_i) - \E[f(\theta, X_1)]$ and notably $\E[\Delta_i(\theta)] = 0$ for any $\theta \in \Theta$.
Consider an algorithm betting a signed fraction of its wealth equal to $B_i(\theta)$ at step $i$ and observing the outcome $\Delta_i(\theta)$.
Note that $B_i(\theta)$ is $\Sigma(X_1, \ldots, X_{i-1}, \theta)$-measurable.
Let the following be the cumulative loss (log-wealth) of the algorithm and the optimal cumulative loss, respectively
\begin{align*}
  \psi_n(\theta) \df \sum_{i=1}^n \ln(1 + B_i(\theta) \Delta_i(\theta))~,\qquad
  \psi_n^{\star}(\theta) \df \max_{\lambda \in \bra{-\frac{1}{1-\mu}, \frac{1}{\mu}} } \ \sum_{i=1}^n \ln\prr{1 + \lambda \Delta_i(\theta)}~.
\end{align*}
We are interested in showing an upper bound on $\int \psi_n^{\star}(\theta) \diff \ps(\theta)$ which holds for all data-dependent distribution $\ps$ simultaneously, and with high probability over the data.
To this end, \citet{OrabonaJ21} show that there exists a betting algorithm that guarantee that for any $\theta \in \Theta$,%
\footnote{Data-dependent bounds on the regret were also shown by \citet{OrabonaJ21}.}
\begin{align}
  \label{eq:pac-bayes-cb-proof-regret}
  \psi_n^{\star}(\theta) - \psi_n(\theta) \leq \ln \frac{\sqrt{\pi} \Gamma(n+1)}{\Gamma(n+\frac12)}~.
\end{align}
So it remains to give a bound on $\psi_n(\ps)$. We will need the following concentration inequality.
\begin{theorem}[{Ville's inequality~\citep[p.~84]{Ville39}}]
  \label{thm:ville}
  Let $\Delta_1, \ldots, \Delta_n$ be a sequence of non-negative random variables such that $\E[\Delta_i \mid \Delta_1, \ldots, \Delta_{i-1}] = 0$.
  Let $M_t > 0$ be $\Sigma(\Delta_1, \ldots, \Delta_{t-1})$-measurable such that $M_0 = 1$, and moreover let $\E[M_t \mid \Delta_1, \ldots, \Delta_{t-1}] \leq M_{t-1}$. 
  Then, for any $\delta \in (0, 1]$,
  $
    \PP\left\{\max_{t} M_t \geq \frac{1}{\delta}\right\} \leq \delta
  $.
\end{theorem}
The proof will also require the following well-known \emph{change-of-measure} inequality:
\begin{lemma}[\cite{DoVa75,DuEl97:weakconv}]
  \label{lem:changeofmeasure}
  Let $p$ and $q$ be probability measures on $\Theta$ such that $p \ll q$.
  Then, for any measurable function $f \,:\, \Theta \rightarrow \reals$, we have
  \[
    \int f(\theta) \diff p(\theta) \leq \KL(p \| q) + \ln \int e^{f(\theta)} \diff q(\theta)~.
  \]
\end{lemma}
Applying the above with $p=\ps$, $q=\pz$, $f=\psi_n$, and taking $\max_{n \in \mathbb{N}}$, we obtain
\begin{align}
  \label{eq:changeofmeasurestep}
  \max_{n \in \mathbb{N}} \int \psi_n(\theta) \diff \ps(\theta) - \KL(\ps \| \pz)
  \leq
  \ln \max_{n \in \mathbb{N}} \underbrace{\int \prod_{i=1}^n \prr{1 + B_i(\theta) \Delta_i(\theta)}  \diff \pz(\theta)}_{M_n},
\end{align}
where we exchanged $\ln$ and $\max_{n \in \mathbb{N}}$.
Now, the plan is to apply \cref{thm:ville} to $M_n$, which requires to show that $M_n$ is a martingale.
Using the notation $\E_i[\cdot] = \E[\cdot \,|\, \Delta_1(\theta), \dots, \Delta_i(\theta), \theta]$, we have
\begin{align*}
  \E_{n-1}[M_n] &=
  \E_{n-1} \int \prod_{i=1}^n \prr{1 + B_i(\theta) \Delta_i(\theta)}  \diff \pz(\theta)
  \stackrel{(a)}{=}
    \int \E_{n-1} \prod_{i=1}^n \prr{1 + B_i(\theta) \Delta_i(\theta)}  \diff \pz(\theta)\\
  &=
    \int
    \E_{n-1}\bra{\prr{1 + B_{n}(\theta) \Delta_{n}(\theta)}} \prod_{i=1}^{n-1} \prr{1 + B_i(\theta) \Delta_i(\theta)}
    \diff \pz(\theta)\\
          &=
    \int
    \prod_{i=1}^{n-1} \prr{1 + B_i(\theta) \Delta_i(\theta)}
    \diff \pz(\theta) = M_{n-1},
\end{align*}
where $(a)$ comes using the fact that $\pz$ is independent from the sample $S$ and by Fubini's theorem.
Thus, applying \Cref{thm:ville} to \cref{eq:changeofmeasurestep}, we obtain
\begin{align*}
  \PP\left\{\max_{n \in \mathbb{N}} \sup_{P \in \cK(\sX^n, \Theta)} \int \psi_n(\theta) \diff \ps(\theta) - \KL(\ps \| \pz)
  \leq  
  \ln \frac{1}{\delta}
  \right\}
  \geq 1-\delta \qquad (\delta \in (0,1])~.
\end{align*}
Finally, using \cref{eq:pac-bayes-cb-proof-regret} gives the statement
and completes the proof.
\jmlrQED

\vspace{-0.2cm}
\section{Experiments}
\label{sec:experiments}
In this section, we validate the numerical tightness of \cref{thm:pac-bayes-cb}. Additional experiments are in \cref{{sec:monte-carlo appendix}}. We perform experiments on simple synthetic scenarios where the parameter space is finite, and we fix the posterior and prior distributions.
We evaluated all the bounds on a sample size range $n\in \{2^c: c=1,\dots,15\}$, and we averaged the bound over 20 repetitions for each sample size.
In particular, we compare \Cref{prop:mc} to several PAC-Bayes baselines such as \citet{McAllester98}, \citet{london2019bayesian}, \citet{maurer2004note} and \citet{tolstikhin2013pac}, and one additional algorithm $\KL$-ver under several synthetic environments.

$\KL$-ver, the KL version of our algorithm, uses $n \cdot \KL(\hat{\mu}_\theta, \mu_\theta)$ for the optimization problem in \cref{eq:main-invert}, instead of $\psi^*_n(\theta, \mu_\theta)$. Theoretically, the log-wealth function $\psi^\star_n (\theta, \mu_\theta)$ is always greater than $n \cdot\KL(\hat{\mu}_\theta, \mu_\theta)$ by Proposition \ref{prop:maurer}.
On the other hand, Maurer's bound is looser than $\KL$-ver, which is shown in the proof of Proposition \ref{prop:maurer} in Appendix~\ref{sec:proofmaurer}.
Hence, $\KL$-ver is an ablation study on our novel optimization problem -- $\KL$-ver is looser than our proposed method of ~\cref{eq:main-invert} but tighter than Maurer's bound that is known to be very tight numerically.
\textfloatsep=.5em

\begin{figure}[t]
\centering
\begin{minipage}[b]{0.5\textwidth}
\centering
\includegraphics[width=\linewidth]{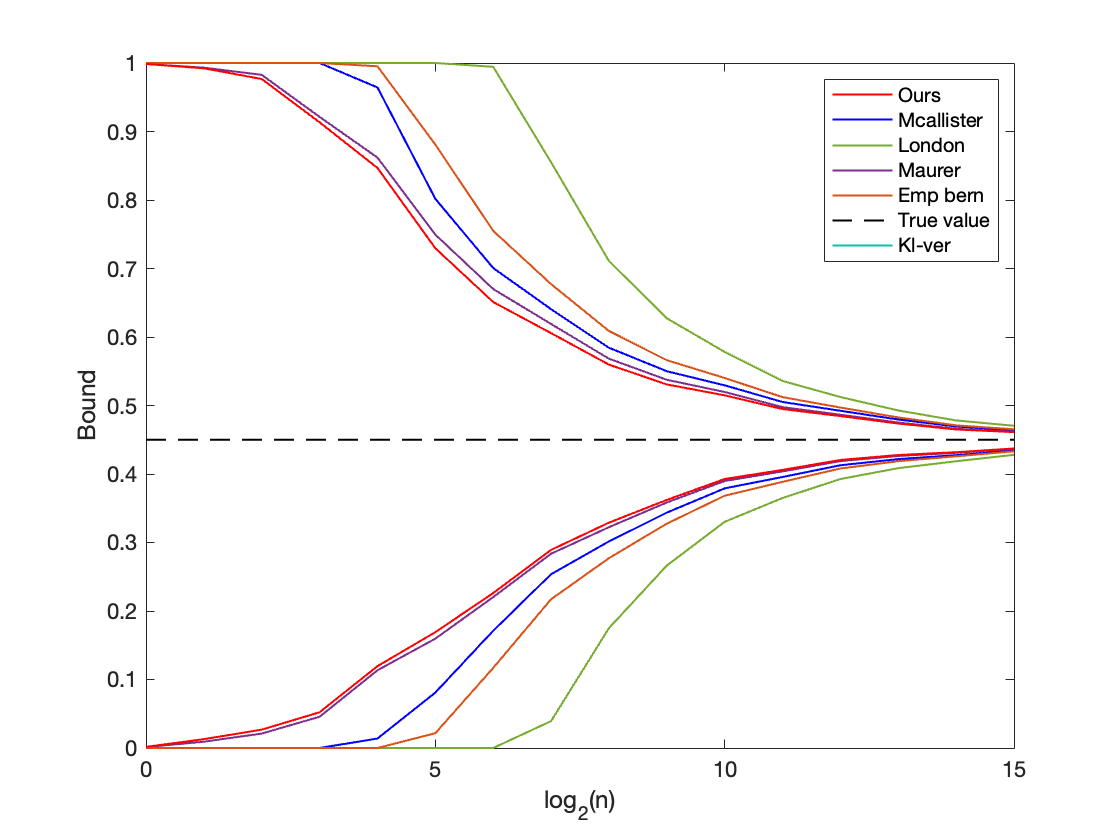}
\vspace{-2em}
\caption{Bernoulli case.}\label{Fig:Bernoulli}
\end{minipage}\begin{minipage}[b]{0.5\textwidth}
  \centering
\includegraphics[width=\linewidth]{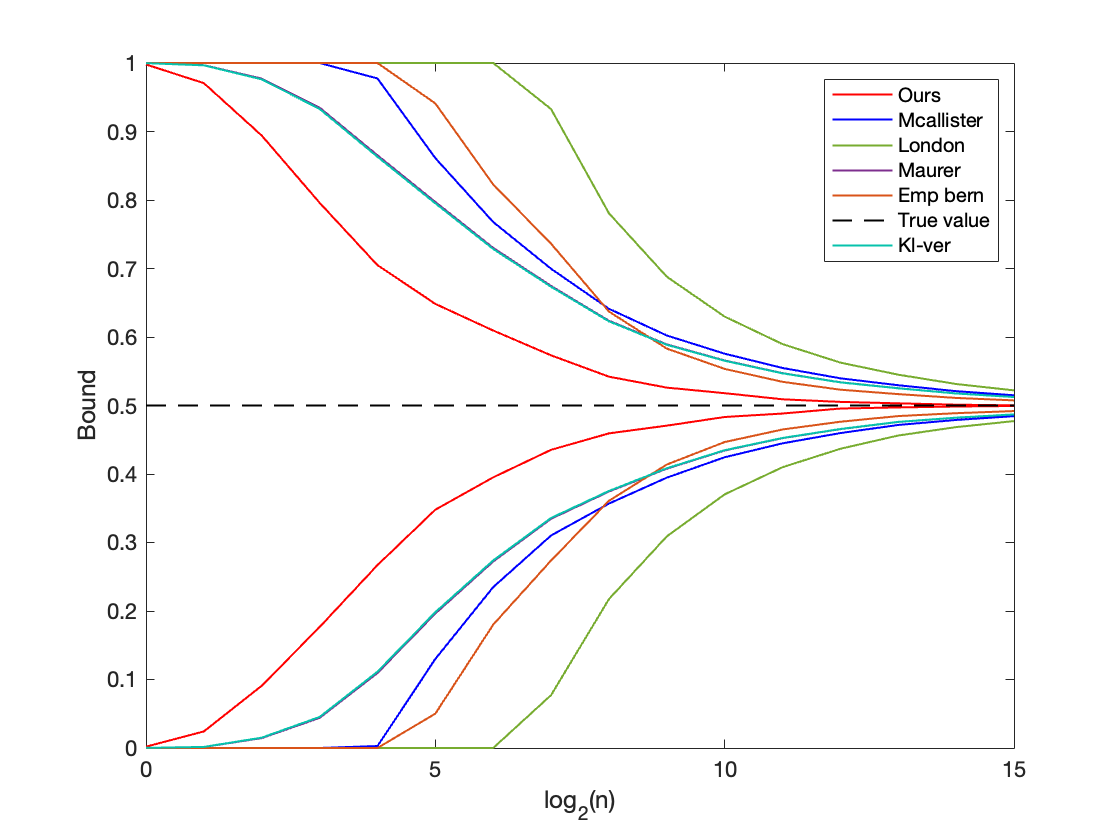}
\vspace{-2em}
  \caption{Binomial case.}\label{Fig:Binomial}
\end{minipage}
\end{figure}

The first experiment, reported in \Cref{Fig:Bernoulli}, represents the case where $X_i \sim \mathrm{Bernoulli}(1/2)$ i.i.d., $\pz = \mathrm{Bernoulli}(0.8)$, $\ps = \mathrm{Bernoulli}(0.9)$, and $f(x, \theta)=x\theta$.
The second experiment, in \Cref{Fig:Binomial}, represents the case where $X_i \sim N(0,1)$ i.i.d., $\pz = (\mathrm{Bin}(6,0.7)-3)/4$, $\ps = (\mathrm{Bin}(6,0.8)-3)/4$, and $f(x, \theta)=(\text{erf}(x \theta)+1)/2$.
Here, $\mathrm{erf}$ is the Gaussian error function and $\mathrm{Bin}(n,p)$ is a binomial distribution, with $n\in \mathbb{N}$ being the number of samples, and $p \in [0,1]$ the probability of success.

Since $\Theta$ is finite, we can explicitly calculate the means without resorting to the Monte Carlo simulation.
Hence, the optimization problem of \Cref{prop:opt} reduces to (and similarly for the lower bound by replacing $\max$ with $\min$)
\begin{align}
  \max_{ \{\mu_\theta: \theta\in \text{supp}(\ps)\}} \  \sum_{\theta\in \Theta} \mu_\theta  \cdot \ps(\theta)
  \quad
  \text{ subject to }  \quad \sum_{\theta \in \Theta} \psi^{\star}_{n}(\theta,\mu_\theta)  \cdot \ps(\theta) \le \sC_n + \ln \frac{1}{\delta}~.
   \label{eq:main-invert-finite-objective}
\end{align}
This is a convex optimization problem as we showed before, so we can use any off-the-shelf solver.\footnote{We use the \texttt{fmincon} function in Matlab.} 

Both figures show that our confidence intervals are consistently tighter than the ones of the baselines. Moreover, our guarantee and the one of $\KL$-ver hold uniformly over time, while it holds for a fixed number of samples for the baselines. Furthermore, for the Bernoulli case, $\KL$-ver is the same as our bound and still better than Maurer's, and in the Binomial case, it is worse than ours and very close to Maurer's bound.
This confirms our theoretical finding that our approach is ``two-inequalities away'' from Maurer's bound. For the case of continuous $\Theta$, check Appendix \ref{sec:monte-carlo appendix}.

\section{Conclusions, limitations, and future work}
We have presented a new PAC-Bayes bound based on a concentration technique derived from the coin-betting formalism. Our new upper bound implies some previous results from PAC-Bayes literature, and at the same time, we have shown that it is tighter in numerical simulations.

One limitation of our result is that it lacks a closed-form minimizer of the upper bound, such as the Gibbs measure in the standard PAC-Bayes analysis. While this is not surprising, it introduces a trade-off between computational complexity and tightness of the bound that was absent in previous approaches. In the future, we aim at precisely characterizing this trade-off, possibly delineating its Pareto frontier.
Another interesting venue is to investigate the numerical minimization of our upper bound over data-dependent distributions for risk minimization problems.

\section*{Acknowledgements}
Francesco Orabona is supported by the National Science Foundation under the grants no. 2022446
``Foundations of Data Science Institute'' and no. 2046096 ``CAREER: Parameter-free Optimization Algorithms
for Machine Learning''.

\bibliography{learning}

\clearpage
\appendix

\section{Remaining proofs}
\label{sec:proofs}
\subsection{Proof of \Cref{prop:mcallister}}
\label{sec:mcallister}
  Consider the right hand side of \cref{eq:coinbettingineq} without $\max_{n \in \mathbb{N}}$.
  Then, we have the following inequalities:
  \begin{align*}
  \int &\max_{\lambda \in [-\frac{1}{1-\mu_{\theta}}, \frac{1}{\mu_{\theta}}]} \ \sum_{i=1}^n \ln\prr{1 + \lambda (f(\theta,X_i) - \mu_{\theta})} \diff \ps(\theta)\\
  &\stackrel{(a)}{\geq}
  \int \max_{\lambda \in [-1, 1]} \ \sum_{i=1}^n \ln\prr{1 + \lambda (f(\theta,X_i) - \mu_{\theta})} \diff \ps(\theta)\\
  &\geq
    \int \max_{\lambda \in [-1, 1]}\Big\{ \lambda \sum_{i=1}^n (f(\theta,X_i) - \mu_{\theta})
    -
    \lambda^2 \sum_{i=1}^n (f(\theta,X_i) - \mu_{\theta})^2 \Big\} \diff \ps(\theta)\\
  &\geq
    \int \max_{\lambda \in [-1, 1]}\Big\{ \lambda \sum_{i=1}^n (f(\theta,X_i) - \mu_{\theta})
    -
    \lambda^2 n \Big\}\\
  &=
    \frac{1}{4 n} \int
    \Big(\sum_{i=1}^n (f(\theta,X_i) - \mu_{\theta})\Big)^2
    \diff \ps(\theta) \tag{Maximizing in $\lambda$; note that optimal $\lambda \in [-1,1]$}\\
  &\geq
    \frac{1}{4 n} \Big(\int
    \sum_{i=1}^n (f(\theta,X_i) - \mu_{\theta})
    \diff \ps(\theta)\Big)^2, \tag{Jensen's inequality}
\end{align*}
where step $(a)$ comes
since $[-1,1] \subset [-\frac{1}{1-\mu_{\theta}}, \frac{1}{\mu_{\theta}}]$ since $\mu_{\theta} \in [0,1]$ almost surely.
Now, applying \cref{thm:pac-bayes-cb} gives
\begin{align*}
  \PP\left\{\frac{1}{4 n} \Big(\int
  \sum_{i=1}^n (f(\theta,X_i) - \mu_{\theta})
  \diff \ps(\theta)\Big)^2
  -
  \sC_n \geq \ln \frac{1}{\delta}\right\} \leq \delta
\end{align*}
and the statement follows.
\jmlrQED
\clearpage
\subsection{Proof of \Cref{prop:maurer}}
\label{sec:proofmaurer}
The proof is based on the following proposition:
\begin{proposition}[{\citep[Proposition 1]{OrabonaJ21}}]
  \label{prop:littlekl}
  Let $X_1, \ldots, X_n \in [0,1]$, let $\hat{\mu} \df (X_1 + \dots + X_n) / n$, and moreover let $\mu \df \E[X_1] \in [0,1]$.
  Then,
  \begin{align*}
    \max_{\lambda \in [-\frac{1}{1-\mu}, \frac{1}{\mu}]} \sum_{i=1}^n \ln\big(1 + \lambda (X_i - \mu) \big) \geq n \, \kl(\hat{\mu}, \mu)~.
  \end{align*}
    Moreover, if $X_1, \ldots, X_n \in \{0,1\}$, we achieve equality in the above.
\end{proposition}
Then, \Cref{prop:littlekl} combined with \Cref{thm:pac-bayes-cb} gives that with probability at least $1-\delta$, simultaneously for all $n \in \mathbb{N}$ and all $\ps$, we have
\begin{align*}
  \frac{\ln \frac{1}{\delta} + \sC_n}{n}
  &\geq \int \kl(\hat{\mu}_{\theta}, \mu_{\theta}) \diff \ps(\theta)\\
  &=  \int \sum_{x \in \{0,1\}} \ln \Bigg(\frac{\mathrm{Bern}(x \mid \hat{\mu}_{\theta})}{\mathrm{Bern}(x \mid \mu_{\theta})} \Bigg) \, \mathrm{Bern}(x \mid \hat{\mu}_{\theta}) \diff \ps(\theta)\\
  &\stackrel{(a)}{\geq}  \sum_{x \in \{0,1\}} \ln \Bigg(\frac{\int \mathrm{Bern}(x \mid \hat{\mu}_{\theta}) \diff \ps(\theta)}{\int \mathrm{Bern}(x \mid \mu_{\theta}) \diff \ps(\theta)} \Bigg) \, \int \mathrm{Bern}(x \mid \hat{\mu}_{\theta}) \diff \ps(\theta)\\
    &=   \kl\Big(\int \hat{\mu}_{\theta} \diff \ps(\theta), \int \mu_{\theta} \diff \ps(\theta) \Big)~,
\end{align*}
where $(a)$ comes by exchanging summation and integration and applying the log-sum inequality.
\jmlrQED
\subsection{Proof of \Cref{thm:empiricalbernstein}}
The proof largely follows that of \citet[Theorem 6]{OrabonaJ21}.
Set $\epsilon_{\theta} = \mu_{\theta} - \hat{\mu}_{\theta}$, and so $\epsilon_{\theta} + \hat{\mu}_{\theta}  \in [0,1]$.
Then, we have
\begin{align*}
  \psi^{\star}(\theta, \mu_{\theta})
  \geq
  \max_{\lambda \in [-1,1]} \ \sum_{i=1}^n \ln\prr{1 + \lambda (f(\theta,X_i) - (\hat{\mu}_{\theta} + \epsilon_{\theta}))},
\end{align*}
and applying Jensen's inequality
\begin{align}
  \label{eq:bernproof1}
  \int \psi^{\star}(\theta, \mu_{\theta}) \diff \ps
  \geq
  \max_{\lambda \in [-1,1]} \ \int \sum_{i=1}^n \ln\prr{1 + \lambda (f(\theta,X_i) - (\hat{\mu}_{\theta} + \epsilon_{\theta}))} \diff \ps~.
\end{align}
Now, we further relax the above by taking a lower bound.
In particular, \citep[Eq. 4.12]{fan2015exponential} shows that
for any $|x| \leq 1$ and $|\lambda| \leq 1$,
\begin{align}
  \label{eq:logineq}
  \ln(1 + \lambda x) \geq \lambda x + \big(\ln(1 - |\lambda|) + |\lambda|\big) x^2~.
\end{align}
The above is combined with the following lemma:
\begin{lemma}[{\citep[Lemma 5]{OrabonaJ21}}]
  \label{lem:empiricalbernsteintechnical}
  Let $f(\lambda) = a \lambda + b \big(\ln(1 - |\lambda|) + |\lambda|\big)$ for some $a \in \reals, b \geq 0$.
  Then, $\max_{\lambda \in [-1,1]} f(\lambda) \geq \frac{a^2}{(4/3) |a| + 2 b}$.
\end{lemma}
Thus,
\begin{align*}
  \int \psi^{\star}(\theta, \mu_{\theta}) \diff \ps
  &\stackrel{(a)}{\geq}
    \lambda \int \sum_{i=1}^n \prr{f(\theta,X_i) - (\hat{\mu}_{\theta} + \epsilon_{\theta})} \diff \ps\\
  &\quad+ \big(\ln(1 - |\lambda|) + |\lambda|\big) \int \sum_{i=1}^n \prr{f(\theta,X_i) - (\hat{\mu}_{\theta} + \epsilon_{\theta})}^2 \diff \ps\\
  &=
    - n \lambda \int \epsilon_{\theta} \diff \ps\\
  &\quad+ \big(\ln(1 - |\lambda|) + |\lambda|\big) \prr{
    \int \sum_{i=1}^n (f(\theta,X_i) - \hat{\mu}_{\theta})^2 \diff \ps + n \int \epsilon_{\theta}^2 \diff \ps
    }\\
  &\stackrel{(b)}{\geq}
    - n \lambda \int \epsilon_{\theta} \diff \ps\\
  &\quad + \big(\ln(1 - |\lambda|) + |\lambda|\big) \prr{
    \int \sum_{i=1}^n (f(\theta,X_i) - \hat{\mu}_{\theta})^2 \diff \ps + n \prr{\int \epsilon_{\theta} \diff \ps}^2
    }\\
  &\stackrel{(c)}{\geq}
    \frac{n^2 \prr{\int \epsilon_{\theta} \diff \ps}^2}{(4/3) n \left| \int \epsilon_{\theta} \diff \ps \right|
    +    
    2 \int \sum_{i=1}^n (f(\theta,X_i) - \hat{\mu}_{\theta})^2 \diff \ps + 2 n \prr{\int \epsilon_{\theta} \diff \ps}^2
    },
\end{align*}
where step $(a)$ comes by application of \cref{eq:bernproof1,eq:logineq}, step $(b)$ comes by Jensen's inequality, and step $(c)$ comes by application of \Cref{lem:empiricalbernsteintechnical}.
Now, the above combined with \cref{thm:pac-bayes-cb} gives
\begin{align*}
  n^2 \prr{\int \epsilon_{\theta} \diff \ps}^2
  \leq
  n \, \prr{\sC_n + \ln \frac{1}{\delta}} \,
  \prr{ \frac43 \left| \int \epsilon_{\theta} \diff \ps \right|
    +    
  2 \hat{V}(\ps) + 2 \prr{\int \epsilon_{\theta} \diff \ps}^2}~.
\end{align*}
Finally, solving the above for $\int \epsilon_{\theta} \diff \ps$, using subadditivity of square root, and relaxing some numerical constants
we get
\begin{align*}
  \abs{\int \epsilon_{\theta} \diff \ps}
  \leq
  \frac{\sqrt{2 \, \prr{\sC_n + \ln \frac{1}{\delta}} \, \hat{V}(\ps)}}{\prr{\sqrt{n} - \frac{2}{\sqrt{n}} \, \prr{\sC_n + \ln \frac{1}{\delta}}}_+}
  +
  \frac{2 \prr{\sC_n + \ln \frac{1}{\delta}}}{\prr{n - 2 \, \prr{\sC_n + \ln \frac{1}{\delta}}}_+}~.
\end{align*}
\jmlrQED
\subsection{Proof of \Cref{lem:maxlogwealthisconvex}}
\label{sec:maxlogwealthisconvex}
We can rewrite $f$ as
\[
f(x) = \max_{0 \leq b \leq 1} \ \ln\left[1+(c-x)\left(-\frac{1}{1-x}+\left(\frac{1}{1-x}+\frac{1}{x}\right)b \right)\right]~.
\]
Now, consider the argument of the max. We claim that it is convex in $x$ for any $b \in [0,1]$.
In fact, the second derivative is
\begin{align}\label{eq:20230202}
   \frac{1}{x^2}+\frac{1}{(1-x)^2}-\frac{(b+c-1)^2}{(-x(b+c)+b c +x)^2}
    = \frac{1}{x^2}+\frac{1}{(1-x)^2}-\frac{1}{(x + \frac{bc}{1-b-c})^2}
\end{align}
We claim that for $b,c\in[0,1]$, we have $g(b,c)\df\frac{bc}{1-b-c} \in (-\infty,-1] \cup [0,\infty)$.
To see this, fix $b$ and consider $c \le 1-b$.
In this regime, $g(b,0)=0$ and $g(b,c)$ is increasing to infinity as $c$ goes from 0 to $1-b$.
In the other regime of $c > 1-b$, we have $g(b,1) = -1$ and $g(b,c)$ is decreasing as $c$ decreases from 1 to $1-b$.
This proves the claim.

Therefore, to lower bound \eqref{eq:20230202} we need to lower bound $(x + z)^2$ where $z \in (-\infty,-1] \cup [0,\infty)$.
Since $(x+z)^2$ is minimized at $-x$ but the range of $z$ never includes $-x$ (except for the boundary case), we have that $\min_{z\in (-\infty,-1] \cup [0,\infty)} (x+z)^2 = \min\{(x-1)^2, x^2\}$.
Therefore,
\begin{align*}
    -\frac{(b+c-1)^2}{(-x(b+c)+b c +x)^2} \ge -\max\left\{\frac{1}{(x-1)^2}, \frac{1}{x^2}\right\} \ge - \frac{1}{(x-1)^2} - \frac{1}{x^2}~.
\end{align*}
Hence, $f$ is a maximum of convex functions, that concludes the proof.
\jmlrQED
\subsection{Proof of \Cref{prop:mcboost}}
\label{sec:mcboost}
It is clear that the optimal log-wealth $\psi_n^{\star}(\th,\mu_\th)$ is non-negative over all arguments, because the maximization range includes 0.
Then, for a fixed block $B_k$, by the non-negativity of $\psi_n^{\star}$, Markov's inequality gives
\begin{align*}
  \PP\left\{ e \int \psi_n^{\star}(\th, \mu_{\th}) \diff \ps(\th) \leq \frac1m \sum_{i \in B_k} \psi_n^{\star}(\th_i, \mu_{\th_{i}}) \right\} \leq 1/e~.
\end{align*}
Since all blocks are independent,
\begin{align*}
  \PP\left\{ \bigcap_{k \in [K]} \Big\{ e \int \psi_n^{\star}(\th, \mu_{\th}) \diff \ps(\th) \leq \frac1m \sum_{i \in B_k} \psi_n^{\star}(\th_i, \mu_{\th_{i}}) \Big\} \right\} \leq e^{-K}
\end{align*}
and so there exists at least one $B_k$ such that with probability at least $1-e^{-K}$,
\begin{align*}
  \int \psi_n^{\star}(\th, \mu_{\th}) \diff \ps(\th) \geq \frac{1}{e m} \sum_{i \in B_k} \psi_n^{\star}(\th_i, \mu_{\th_{i}})~.
\end{align*}
Thus, with probability at least $1-e^{-K}$,
\begin{align*}
  \int \psi_n^{\star}(\th, \mu_{\th}) \diff \ps(\th) \geq \min_{k \in [K]} \ \frac{1}{e m} \sum_{i \in B_k} \psi_n^{\star}(\th_i, \mu_{\th_{i}})~.
\end{align*}
\jmlrQED

Note that we can change $e$ to any other constant $C>1$. Then, with probability $1-C^{-K}$, 
\begin{align*}
  \int \psi_n^{\star}(\th, \mu_{\th}) \diff \ps(\th) \geq \min_{k \in [K]} \ \frac{1}{C m} \sum_{i \in B_k} \psi_n^{\star}(\th_i, \mu_{\th_{i}})~.
\end{align*}
\subsection{Proof of \Cref{prop:mc}}
\label{sec:mc}
Recall that $(\th_i)_{i \in B_k} \sim \ps^m$ for any $k \in [K]$ and $\mu_{\th_i} = \EE[f(\th_i, X_1)], \forall i\in [Km]$.
Our goal is to find an upper and lower bound for $\mu^* := \int \mu_\th \dif \ps(\th)$.
Given $\ps(\th)$, \Cref{thm:pac-bayes-cb} gives us
\begin{align}\label{eq:pac-bayes-cb-local}  \PP\left\{  \exists n \in \mathbb{N}, \exists P_n: \ \int \psi_n^{\star}(\th,\mu_\th) \dif \ps(\th) - \sC_n \le  \ln\frac{1}{\delta}\right\} \ge 1-\dt~.
\end{align}
Moreover, using the union bound and the standard KL-divergence concentration inequality for $[0,1]$-bounded random variable, we have
\begin{align}
    \PP\left\{\max_{k\in[K]} \ \kl\del{\fr1m \sum_{i \in B_k} \mu_{\th_i},~ \mu^*}\le \ln\fr{K}{2\delta} \right\} \ge 1 - \dt ~.
\end{align}
Furthermore, \Cref{prop:mcboost} states that
\begin{align}\label{eq:mcboost}   \PP\left\{ \min_{k \in [K]}\ \fr1m\sum_{i \in B_k} \psi_n^{\star}(\th_i, \mu_{\th_i}) \leq \int \psi_n^{\star}(\th,\mu_\th) \dif \ps(\th) \bmid X_1, \dots, X_n\right\} \ge 1 - e^{-K}~.
\end{align}
By the union bound, and assuming setting $K=\lceil \ln(1/\delta) \rceil$ one can see that with probability at least $1-3\dt$, the concentration events in the three inequalities are all true.
Assume that these events are true.
It suffices to show $\mu^* \le M_U$ since the proof of $\mu^* \ge M_L$ is symmetric.

Denote by $\nu^U_{\th_i}$ the solutions of the optimization problem in \eqref{eq:optimub0}.

Let $k^* = \arg \min_{k\in[K]} \ \fr1m\sum_{i \in B_k} \psi_n^{\star}(\th_i, \mu_{\th_i})$.
By the events in~\eqref{eq:pac-bayes-cb-local} and~\eqref{eq:mcboost}, we have that $(\mu_{\th_i})_{i\in B_{k^*}}$ is a feasible solution of~\eqref{eq:optimub0} with $k=k^*$.
This implies that 
\begin{align}\label{eq:optim-consequence}
    \bar\mu(k^*) := \fr1m \sum_{i\in B_{k^*}} \mu_{\th_i}\le \fr1m \sum_{i\in B_{k^*}} \nu^U_{\th_i} \le \fr1m \sum_{i\in B_{\hk_U}} \nu^U_{\th_i} = \bar\nu_U(\hk_U) =: \nu^*.
\end{align}
Consider the following two cases:
\begin{itemize}
    \item Case 1: $\nu^* \ge \mu^* $.
    \\
    We just need to verify that $\nu^* \le M_U$. This holds by the definition of $M_U$.
    \item Case 2: $\nu^* < \mu^*$. 
    \\
    We have     \begin{align*}
        \kl(\nu^*, \mu^*) \le \kl(\bar\mu(k^*), \mu^*) \le \ln\fr{K}{2\dt} \stackrel{\eqref{eq:MU}}{=} \kl(\nu^*, M_U)~.
    \end{align*}
    where the first inequality is due to $\bar\mu(k^*) \le \nu^* < \mu^*$.
    Applying the monotonicity of $\kl$ in the second argument to  $\kl(\nu^*, \mu^*) \le \kl(\nu^*, M_U)$, we have $\mu^* \le M_U$.
\end{itemize}
This concludes the proof.
\jmlrQED

\section{Monte Carlo experiment}
\label{sec:monte-carlo appendix}
In this section, we investigate numerically the Monte Carlo approximation discussed in Section~\ref{subsubsec: monte-carlo approximation}. In particular, we want to validate the claim that the confidence intervals calculated with Algorithm~\ref{alg:mcapprox} are better than the ones of \cite{maurer2003bound} when enough Monte Carlo sample are used.

Here is the list of parameters in our experiment:
\begin{itemize}
    \item The number of samples: $n=32$ (Figure \ref{fig:monte_carlo change m}), $n \in \{2^c: c\in \{1,\dots,4\} \}$ (Figure \ref{fig:monte_carlo change n large m}).
    \item The number of groups: $K=4$, the corresponding multiplier: $C=2.1147$ (check the end of Appendix \ref{sec:mcboost}). 
    \item The number of Monte Carlo samples in each group: $m \in \{2^c: c=1,\dots, 10 \}$ (Figure \ref{fig:monte_carlo change m}), $m=256$ (Figure \ref{fig:monte_carlo change n large m}).
    \item Parameter space $\Theta = \mathbb{R}$.
    \item Prior distribution $P_0$ is $N(0,1)$.
    \item Posterior distribution $P_n$ is $N(0,0.25)$.
    \item Samples: $X_1, \dots, X_n$ drawn from $N(0,1)$, i.i.d.
    \item $f(x,\theta)=(\mathrm{erf}(x\theta)+1)/2$.
    \item Failure probability $\dt=0.05$.
\end{itemize}

For the fair use of parameters, we ensure that the total number of MC samples used in Maurer's bound matches that of ours.
That is, when ours use $m$ Monte Carlo samples for each group $k\in[K]$, Maurer's bound use $M=m\cdot K$ Monte Carlo samples.
We describe how we compute Maurer's bound numerically in Algorithm~\ref{alg:maurer_monte_carlo}.

\begin{algorithm}
\caption{Monte Carlo Approximation for Maurer's Bound}
\label{alg:maurer_monte_carlo}
\begin{algorithmic}[1]
\STATE \textbf{Input}: Failure probability $\dt$, sample size parameters $M \in \NN$, posterior $\ps$, prior $\pz$
\STATE Sample $(\th_i)_{i \in [M]} \sim \ps^M$
\STATE $\hat{\mu}_M = \frac{1}{M} \sum_{j=1}^M \hat{\mu}_{\theta_i} = \frac{1}{Mn} \sum_{j=1}^M \sum_{i=1}^n f(X_i, \theta_i)$
\STATE Let $w = \sqrt{\frac{\ln(\frac{2}{\delta})}{2M}}$
\STATE $\hat k_U = \hat{\mu}_M +w$ and $\hat k_L = \hat{\mu}_M -w$
\STATE Compute
\vspace{-.5em}
\begin{align}\label{eq:Maurer22}
    M_U \!=\! \max\left\{\mu\!: \kl\left( \hat k_U, \mu\right)\le\! \frac{\mathcal{C}_n+\ln \frac{2}{\delta}}{n}\right\}~,\,
    M_L \!=\! \min\left\{\mu\!: \kl\left( \hat k_L, \mu\right)\le\! \frac{\mathcal{C}_n +\ln \frac{2}{\delta}}{n}\right\}
\end{align}
\STATE \textbf{Output:} $M_L$ and $M_U$
\end{algorithmic}
\end{algorithm}

The value $w$ in \Cref{alg:maurer_monte_carlo} is the confidence width of the Monte Carlo error between $\hat{\mu}_m$ and $\int \hat \mu_\theta \diff \ps (\theta)$ based on the following Proposition \ref{prop:monte-carlo width}.

\begin{proposition}[Monte Carlo error bound]
  \label{prop:monte-carlo width}
  Let $\delta \in (0,1]$. Given the samples $X_1, \dots, X_n$,
  with probability at least $1-\delta$ over the Monte Carlo samples $(\theta_i)_{i \in [M]}$, we have
  \begin{align*}
  \frac{1}{M}\sum_{i=1}^M \hat\mu_{\theta_i} - \int \hat{\mu}_{\theta} \diff \ps(\th)
  \leq
    \sqrt{\frac{\ln (\frac{1}{\delta})}{2M}}~.
\end{align*}
\end{proposition}
\begin{proof}
When the set of samples $\{X_i\}_{i=1}^n$ is given, we can consider $\{ \hat{\mu}_{\theta_i}\}_{i=1}^M$ as i.i.d. samples from a distribution supported on $[0,1]$. Therefore, by Hoeffding's inequality, 
\[
\PP \left\{\frac{1}{M} \sum_{i=1}^M \hat{\mu}_{\theta_i} - \int \hat{\mu}_\theta \diff \ps (\theta) \geq \sqrt{\frac{\ln (\frac{1}{\delta})}{2M}}\right\} \leq \dt ~.
\]
\end{proof}

When the parameter space $\Theta$ is discrete, we can precisely compute $\int \hat \mu_{\theta}\diff \ps (\theta)$, and from this value we can use Maurer's bound to obtain the confidence bound. However, in our current setting we only have $[\hat k_L, \hat k_U]$ that is a confidence interval for $\int \hat{\mu}_{\theta}\diff \ps (\theta)$. In this case, the confidence bound of actual $\int \mu_\theta \diff \ps (\theta)$ is the union of all possible confidence intervals, or formally, 
\[
\cup_{\mu' \in [\hat k_L, \hat k_U]} \left\{ \mu: \kl(\mu', \mu)\leq \frac{\mathcal{C}_n + \ln \frac{1}{\delta}}{n}\right\}~.
\]

Thanks to the property of the $\kl$, we only need to check two endpoints $\hat k_L$ and $\hat k_U$ to get the final Maurer's bound $(M_L , M_U)$.

\begin{figure}[t]
\centering
\begin{minipage}[b]{0.48\textwidth}
\includegraphics[width=\linewidth]{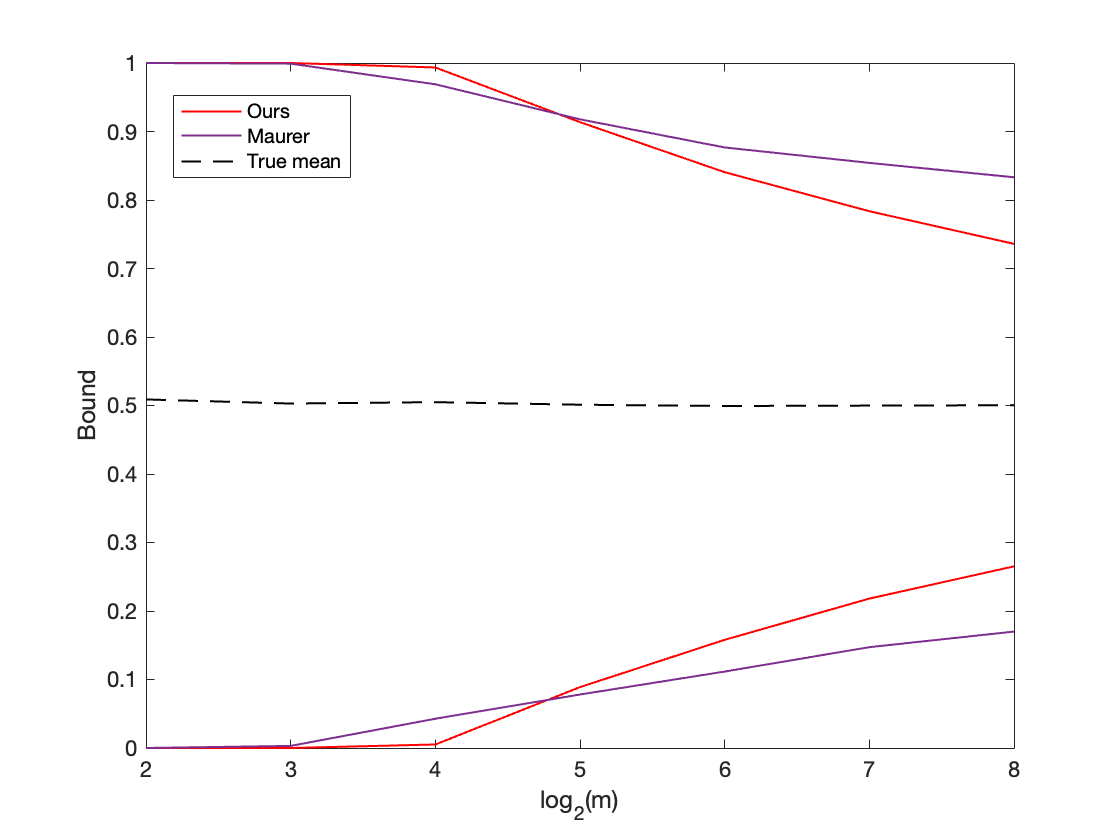}
\caption{Confidence width vs Monte Carlo samples ($m$).}
\label{fig:monte_carlo change m}
\end{minipage}%
\hfill
\begin{minipage}[b]{0.48\textwidth}
\includegraphics[width=\linewidth]{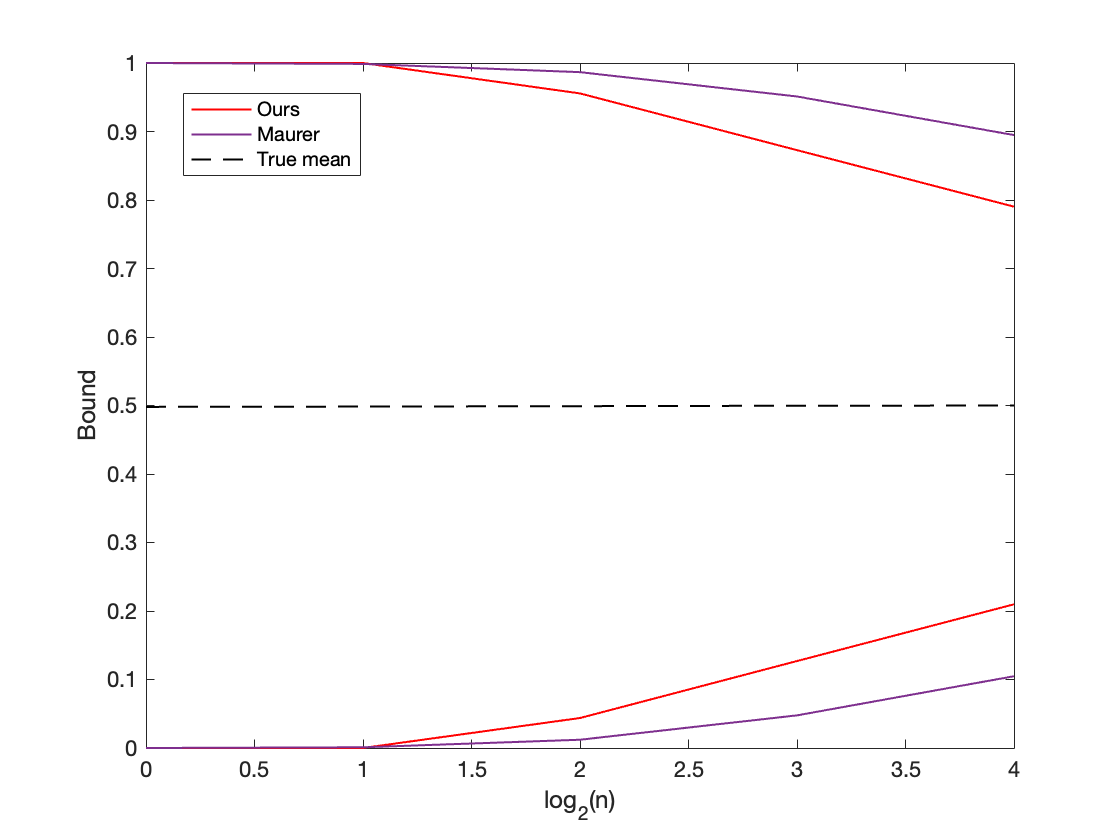}
\caption{Confidence width vs the number of samples ($n$).}
\label{fig:monte_carlo change n large m}
\end{minipage}
\end{figure}


Figure \ref{fig:monte_carlo change m} shows that our Algorithm \ref{alg:mcapprox} outperforms the Maurer's bound with sufficient amount of Monte Carlo samples.
Even with the increasing number of data samples (Figure \ref{fig:monte_carlo change n large m}), our algorithm steadily outperforms the Maurer's bound. 



\end{document}